\documentclass[twoside,11pt]{article}

\usepackage{jmlr2e}

\usepackage{amsmath}
\usepackage{bm}
\usepackage{bbold}
\usepackage{booktabs}
\usepackage{float}
\usepackage{multirow}
\usepackage{numprint}
\usepackage{blkarray}
\usepackage[geometry]{ifsym}
\usepackage{subcaption}
\usepackage{wrapfig}

\newcommand\numberthis{\stepcounter{equation}\tag{\theequation}}
\newcommand{\T}{\mathsf{T}}

\DeclareMathOperator*{\argmin}{argmin}
\DeclareMathOperator{\sgn}{sgn}
\DeclareMathOperator{\nullspace}{null}
\DeclareMathOperator{\rowspace}{row}%L

\ShortHeadings{gLOP: the global and Local Penalty for Capturing Predictive Heterogeneity}{Rose, Rhiannon and Lizotte, Daniel}
\firstpageno{1}

\begin{document}

\title{gLOP: the global and Local Penalty\\
for Capturing Predictive Heterogeneity}

\author{\name Rhiannon V. Rose \email rrose24@uwo.ca \\
       \addr Department of Epidemiology \& Biostatistics\\
       Western University\\
       London, ON, Canada 
       \AND
       \name Daniel J. Lizotte \email dlizotte@uwo.ca \\
       \addr Department of Computer Science, Department of Epidemiology \& Biostatistics \\
       Western University\\
       London, ON, Canada }

\maketitle

\begin{abstract}
When faced with a supervised learning problem, we hope to have rich enough data to build a model that predicts future instances well. However, in practice, problems can exhibit {\em predictive heterogeneity}: most instances might be relatively easy to predict, while others might be {\em predictive outliers} for which a model trained on the entire dataset does not perform well. Identifying these can help focus future data collection. We present gLOP, the global and Local Penalty, a framework for capturing predictive heterogeneity and identifying predictive outliers. gLOP is based on penalized regression for multitask learning, which improves learning by leveraging training signal information from related tasks. We give two optimization algorithms for gLOP, one space-efficient, and another giving the full regularization path. We also characterize uniqueness in terms of the data and tuning parameters, and present empirical results on synthetic data and on two health research problems.
%at this length, 984 chars including spaces.
\end{abstract}

\section{Introduction}

We are motivated by prediction problems in healthcare where we have data about a population of patients, but only limited data about each individual patient. As an example, we will present a problem where the goal is to use a self-reported scale of depressive symptoms to predict a clinician-rated scale which has been deemed more suitable for decision-making. We have data on hundreds of different patients, but no more than 15 paired observations per patient. We would like to determine whether a single model relating self-reported scores to clinician-rated scores predicts all patients adequately, or if there are patients for whom such a model needs tailoring to work well for them.
We also present another problem where the goal is to predict Parkinson's disease symptom progression from speech waveform features. Both of these problems have potential applications in telehealth. In studies of depression, although self-reported symptom scores are often collected, clinician-rated scores are commonly used to support decisions about altering a patient's treatment plan \citep{rush04sequenced}. For patients who may not have regular access to a trained assessor, the combination of a self-reported score and a prediction of the clinician-rated score could enable more timely revision of their treatment plan. The Parkinson's disease study had a similar intent; effective predictive models would help remote patients better track their symptoms and inform any revisions to their treatment plan. 

In both of these problems (and in many others) we expect a degree of {\em predictive heterogeneity.} That is, we expect that for each patient, we could build an effective predictive model given enough patient-specific examples, but: 1) we do not have access to enough data per patient and 2) we suspect that individual heterogeneity will preclude building a ``global model'' that works well for all patients. We expect that some patients will be {\em predictive outliers} whose models differ substantially from the norm. While this cannot directly help us build better models for new patients, if we find evidence that predictive outliers exist, then we may be able to improve prediction by gathering more features on the patients we already have in order to help distinguish such outliers {\em a priori} in future models, and thus improve predictive power. 
If we see no evidence for predictive outliers, a better strategy might be to gather more training examples rather than more features on the existing examples.

Our contributions are as follows. We present the global and Local Penalty (gLOP) model, which learns a predictive model with a {\em global} component that applies to all patients and a {\em local} component that captures individual variation, while performing simultaneous feature selection for both. We describe in detail how it is related to previous approaches, and we present two optimization techniques for gLOP, one that is very space-efficient and one that provides the entire regularization path. We characterize the conditions under which the gLOP estimate is unique. We provide empirical evidence that gLOP has better in-population predictive performance than previous approaches. Finally, we show how gLOP can be used to detect predictive outliers by applying it to two health research problems.

\section{Background}
\newcommand{\y}{\ensuremath{\boldsymbol{y}}}
\newcommand{\g}{\ensuremath{\boldsymbol{g}}}
\newcommand{\bbeta}{\ensuremath{\boldsymbol{\beta}}}

Lowercase letters (e.g.\ $c, \alpha$) denote scalars, bold lowercase letters denote vectors (e.g.\ $\boldsymbol y$), and uppercase letters denote matrices (e.g.\ $L$). Superscripts index elements of a list (e.g.\ $L^k$ is the $k$th matrix in a list) and subscripts index matrix columns (e.g.\ $L^k_j$ is the $j$th column of the $k$th matrix) or vector elements (e.g.\ $\alpha_j$ is the $j$th element of the vector $\boldsymbol\alpha$.)

\subsection{Penalization and the Lasso}

Both gLOP and its related methods are based on feature selection through {\em penalization} or {\em regularization}. Given an $n \times p$ design matrix of predictor variables $X$, and a binary $n \times 1$ response variable $\y$, the general form of a penalized regression problem is
$\hat\bbeta = \argmin_{\bbeta} \mathcal{L}(X,\y) + \mathcal{P}(\bbeta)$,
where $\mathcal{L}$ is a loss function measuring how well the model $\bbeta$ fits the data, and $\mathcal{P}$ is a penalty term on the complexity of $\bbeta$. The \textit{Least Absolute Shrinkage and Selection Operator} (lasso) \citep{tibshirani1996} is one such method, whose parameter estimates are given by
$
\hat{\bbeta}^{\text{lasso}}(\lambda) = \argmin_{\bbeta}  \tfrac{1}{2} \| \y - X\bbeta\|^2_2 + \lambda\|\bbeta\|_{1}
$,
where $\lambda$ controls the amount of shrinkage of the estimates. When $\lambda = 0$, the model is unpenalized  with all features present; higher values of $\lambda$ will cause more coefficients of $\bbeta$ to be shrunk to 0, giving a sparser model that includes only the most relevant features.

\subsection{Penalization Methods for Multi-task Learning}

Multi-task learning methods \citep{caruana1997} were originally described as learning from data about a large number of related ``tasks'' when there is only limited data about each individual task. This framework has a clear connection to the problem we face with limited patient data. We make one terminological change in this work: We use the word ``patient'' instead of ``task,'' to clarify that there is only one predictive ``task,'' e.g., we are always trying to predict a symptom score. The data are divided into different patients from which we have observations relevant to performing the task. 

%The \textit{group lasso} allows for pre-specified groups of variables to be considered together for inclusion or exclusion in a given model. 
%We will see below that this idea can be used to allow parameters for multiple patients to enter the model together in the multi-task learning framework. 
%Given an $n \times p$ design matrix of predictor variables $X$, a binary $n \times 1$ response variable $y$, and a list of $\kappa$ vectors $\Grp^1,...,\Grp^\kappa$ of group indicator variables, the group lasso is defined by \citet{yuan2006} as $
%\hat{\bbeta}^{\text{group}}(\lambda,\Grp^{1}\!\!,\!...,\Grp^\kappa) = \argmin_{\bbeta} \tfrac{1}{2} \| y - X\bbeta \|^2_2 + \lambda\sum\limits_{k=1}^\kappa\|\bbeta_{\Grp^k}\|_2
%$
%where $\bbeta_{\Grp^k}$ is the sub-vector of $\bbeta$ whose corresponding group indicators in $\Grp^k$ are on. The penalty for the group lasso is a hybrid between the lasso's $\ell_1$ penalty, which encourages individual groups to be selected sparsely, and the $\ell_2$ penalty of ridge regression, which shrinks but does not encourage sparsity {\em within} groups. The resulting penalty induces sparse group-level feature selection. 
\newcommand{\Nrm}{\ensuremath{\boldsymbol{\nu}}}
\newcommand{\nrm}{\ensuremath{\nu}}
The \textit{Composite Absolute Penalties} (CAP) family of models, of which the lasso is a special case, are used in various cases of hierarchical and group-based feature selection.  Intuitively, CAP penalties work by taking the norms of vectors that contain coefficients of different groups of variables, and then penalizing the norm of the vector containing each of the group norms. Construction of a general CAP works as follows.
For each of $\kappa$ groups of coefficients, we create sub-vectors of coefficients denoted $\bbeta^k$. We take the norms $\nrm_k = \|\bbeta^k\|_{\gamma_k}$ of each of these sub-vectors, and place them in a $\kappa$-dimensional vector $\Nrm = (\nrm_1,...,\nrm_\kappa)$. The CAP penalty is given by $\|\Nrm\|_{\gamma_0}^{\gamma_0} = \sum\limits_k |\nrm_k|^{\gamma_0}$. Using a least-squares loss, the CAP estimate as a function of $\lambda$ is given by
$\hat{\bbeta}^{\text{CAP}}(\lambda, \boldsymbol{\gamma}) =  \argmin_{\bbeta} \tfrac{1}{2} \| y - X\bbeta \|^2_2  + \lambda \|\Nrm\|_{\gamma_0}^{\gamma_0}$.
To induce sparsity across groups, we would select $\gamma_0 = 1$.
%, and to ensure that features within a group enter the model at the same time we would set $\gamma_k > 1$ for $k \in \{1,...,\kappa\}$. 
CAPs are convex provided that the norms used are also convex, making global optimization feasible \citep{zhao2009}.

%eature selection techniques like the lasso and group lasso work under the assumption that adequate amounts of independent and identically distributed training instances are available.
%In particular, they are not designed to take advantage of the additional information that might be shared across ``tasks,'' or as is relevant in the context of healthcare prediction, patients. 
In order to leverage data collected from different (but similar) tasks (we would say patients), \citet{jalali2010} developed a CAP-based method used for what they term \textit{dirty data}: data containing features that are not relevant to all tasks. 
They outlined a \textit{dirty model} that can leverage similarity between tasks by identifying features that are globally relevant, but that can also allow for the inclusion of features that are only relevant in some tasks. 
This corresponds to the situation in which most patients may be well-served by the same predictive model, but some patients require different models. Let $n_k$ and $p$ denote the number of examples and features, respectively, {\em per patient.} There are $\kappa$ patients. The dirty model is parameterized by a ``global'' matrix of parameters $B$ and a ``local'' matrix of parameters $S$, both of which are $p \times \kappa$. 
The feature matrix for patient $k$ is denoted by $X^k$, and the targets are denoted $\y^k$. 
The optimization problem for the dirty model is\footnote{Here we give the objective function as defined in the code provided by \citeauthor{jalali2010}. It differs from the objective stated in their paper by a factor of $\tfrac{1}{2n}$ applied to the squared loss term.}
\begin{equation}
\argmin_{B,S}\! \sum_{k=1}^\kappa  \|\y^{k}-X^{k}(B_k+S_k)\|_2^2 +\lambda_B\|B\|_{1,\infty} + \lambda_S\|S\|_{1,1}. \label{eq:dirtymodel}
\end{equation}
The learned parameters for patient $k$ are then given by $\bbeta^k = \hat{B}_k + \hat{S}_k$, where $\hat{B}$ and $\hat{S}$ are solutions of (\ref{eq:dirtymodel}). The dirty model applies an $\ell_{1,\infty}$ norm penalty to parameter matrix $B$, which is given by 
$\|B\|_{1,\infty} = \sum_j \|(B^\T)_j\|_{\infty}%, \text{ where } \|B_j\|_{\infty} = \max_k \lvert b_{k,j} \rvert.
$.
 Hence, this is a CAP problem that uses the $\ell_1$ and $\ell_\infty$ norms to achieve group sparsity.
The effect of this penalty is to induce entire rows of $B$ to enter the model at the same time, that is, as soon as one element of $B$ enters the model, all elements corresponding to the same feature in different patients may enter the model with no additional penalty as long as their absolute value stays less than or equal to that of the largest element. 
Thus parameters in $B$ will ``turn on'' for all patients at once. Note however that there is no strict enforcement of equality across the rows of $B$, so while the selection is global, the actual parameter values are not. 
The secondary parameter matrix $S$ is penalized using the $\ell_{1,1}$ norm, which is simply the sum of the absolute values of the elements in the matrix; this induces element-wise sparsity in $\hat{S}$ and allows individual patients to ``turn on'' additional features if necessary. 
A schematic view of a hypothetical learned $\hat B$ and $\hat S$ is shown in Figure \ref{sampleDM}.
\begin{figure}
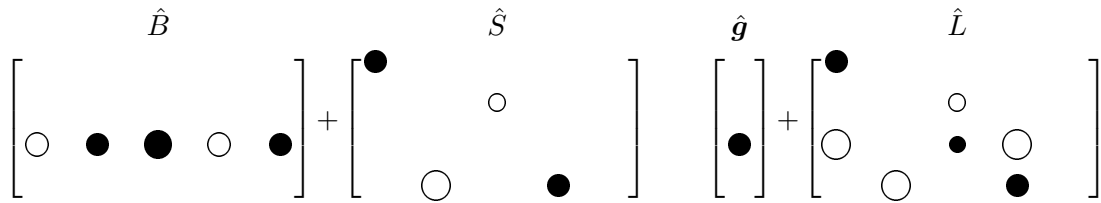
\centering
\begin{subfigure}[t]{0.57\textwidth}
\[
\renewcommand\BAextrarowheight{2pt}
\begin{blockarray}{*{5}{c}}
& & \hat B & & \\
\begin{block}{[*{5}{c}]}
 \phantom{\BigCircle} & \phantom{\BigCircle}& \phantom{\BigCircle}& \phantom{\BigCircle}  \\
 &  &  &  & \\
\mbox{\Circle} & \mbox{\FilledCircle} & \mbox{\FilledBigCircle} & \mbox{\Circle} & \mbox{\FilledCircle} \\
 &  &  &  & \\
\end{block}
\end{blockarray}
+
\renewcommand\BAextrarowheight{2pt}
\begin{blockarray}{*{5}{c}}
& & \hat S & & \\
\begin{block}{[*{5}{c}]}
\mbox{\FilledCircle} &  &  &  & \\
&  & \mbox{\SmallCircle} &  & \\
\phantom{\BigCircle}& \phantom{\BigCircle}& \phantom{\BigCircle}& \phantom{\BigCircle}& \phantom{\BigCircle}\\
& \mbox{\BigCircle} &  & \mbox{\FilledCircle} & \\
\end{block}
\end{blockarray}
\]
\vspace{-2em}
\caption{Example estimated dirty model structure with global parameters $\hat{B}$ and local parameters $\hat{S}$.}
\label{sampleDM}
\end{subfigure}%
~~~~
\begin{subfigure}[t]{0.36\textwidth}
\[
\renewcommand\BAextrarowheight{2pt}
\begin{blockarray}{*{1}{c}}
\hat\g \\
\begin{block}{[*{1}{c}]}
 \\
  \\
\mbox{\FilledCircle}\\
 \\
\end{block}
\end{blockarray}
+
\renewcommand\BAextrarowheight{2pt}
\begin{blockarray}{*{5}{c}}
& & \hat L & & \\
\begin{block}{[*{5}{c}]}
\mbox{\FilledCircle} & \phantom{\BigCircle}& \phantom{\BigCircle}& \phantom{\BigCircle}& \phantom{\BigCircle} \\
\phantom{\BigCircle} &  & \mbox{\SmallCircle} &  & \\
\mbox{\BigCircle} && \mbox{\FilledSmallCircle} &  \mbox{\BigCircle}  & \\
& \mbox{\BigCircle} &  & \mbox{\FilledCircle} & \phantom{\BigCircle}\\
\end{block}
\end{blockarray}
\]
\vspace{-2em}
\caption{Example estimated gLOP model structure with global parameters $\hat{g}$ and local parameters $\hat{L}$.}
\label{samplegLOP}
\end{subfigure}
\caption{Schematic view of a dirty and a gLOP model that would make identical predictions. Columns of $\hat B$, $\hat S$, and $\hat L$ represent coefficients for each patient. Filled circles represent positive coefficients, empty circles represent negative. Sizes indicate magnitude. }
\end{figure}
Representing the parameters of the model as the matrices $B$ and $S$ allows the dirty model to capture some similarities between patients while still allowing for individual variation. 
However, the interpretation of the model is not straightforward because coefficients for a single feature in $B$, although considered the ``global'' part of the model, are not required to be the same for all patients. 
Furthermore, although the dirty model penalty is convex, it does not admit a straightforward regularization path algorithm.

\section{gLOP}
We now introduce our global and LOcal Penalty (gLOP) model. As for the ``dirty model,'' we learn global and local parameters, but our decomposition uses a \emph{vector} $\g$ (for ``global'') and a matrix $L$ (for ``local''). The $p\times 1$ vector $\g$ contains global coefficients that apply to all patients, and the columns of the $p\times \kappa$ matrix $L$ 
contain local coefficients that apply only to their specific patients. 
This makes our model more easily interpretable, as the global effects are clearly distinguishable from individual effects and are enforced to be {\em the same across all patients.} Thus we apply a simpler $\ell_1$ penalty to $\g$ instead of the $\ell_{1,\infty}$ norm used in the dirty model\ because we do not need to use the penalty to ``push'' the global parameters to be the same across patients. The gLOP optimization problem given by
\begin{equation} \label{eq:gLOP}
\argmin_{\g,L}  \sum_{k=1}^\kappa \tfrac{1}{2n_k} \|\y^{k}-X^{k}(\g + L_k)\|_2^2 + \lambda_{\g} \|\g\|_1 + \lambda_L\|L\|_{1,1}.
\end{equation}
The learned parameters for patient $k$ are then given by $\bbeta^k = \g + \hat{L}_k$, where $\hat{\g}$ and $\hat{L}$ are given by (\ref{eq:gLOP}).
A diagram of an example $\hat \g$ and $\hat L$ is shown in Figure \ref{samplegLOP}.

Simplifying the matrix $B$ into our vector $\g$  is advantageous as it reduces the number of parameters, which reduces the potential for overfitting compared to the dirty model. Additionally, because the global coefficients are identical across patients, the model is easily interpretable for all patients together and individual patients, increasing the utility of the model in scientific practice. This model formulation also offers computational advantages; we present a space-efficient block coordinate minimization with the lasso as a subroutine, and we present a method that allows us to compute the full regularization path, which is not possible with the dirty model. Finally, the gLOP formulation allows us to establish deterministic conditions for uniqueness in terms of the data and the penalty parameters.

\subsection{Block Coordinate Minimization}
Problem (\ref{eq:gLOP}) can be optimized using block coordinate minimization \citep{wright1999numerical} using the standard lasso. A significant advantage of this method is its use of the lasso as a subroutine, for which fast implementations are commonly available. To solve (\ref{eq:gLOP}) using the lasso, we decompose the optimization into separate problems for $L$ and $\g$. If we fix $L$, the $\g$ that optimizes (\ref{eq:gLOP}) is given by
\begin{equation}
\argmin_{\g}  \tfrac{1}{2n_k}  \|\tilde{\y} - \tilde{X} \g\|_2^2 + \lambda_{\g}\|\g\|_1,~
\mathrm{where}
~\tilde{X} =
\left[ \begin{array}{cc} 
X^{1}\\
X^{2}\\
\vdots \\
X^{\kappa}\\
\end{array} \right]
\tilde{y} =
\left[ \begin{array}{cc} 
\y^{1} - X^{1} L_1\\
\y^{2} - X^{2} L_2\\
\vdots \\
\y^{\kappa} - X^{\kappa} L_{\kappa}\\
\end{array} \right].
\label{eq:solveforg}
\end{equation}
Problem (\ref{eq:solveforg}) is a standard lasso problem. If we fix $\g$, the $L$ that optimizes (\ref{eq:gLOP}) is given by
\begin{align*}
%& \argmin_{L}  \sum_{k=1}^\kappa \tfrac{1}{2n_k} \|\y^{k}-X^{k}(\g + L_k)\|_2^2 + \lambda_{\g}\|\g\|_1 + \lambda_L\|L\|_{1,1} \\
% & \argmin_{L}  \sum_{k=1}^\kappa \left[ \tfrac{1}{2n_k} \|\y^{k}-X^{k}(\g + L_k)\|_2^2 \right] + \lambda_L\|L\|_{1,1} \\
%= & \argmin_{L}  \sum_{k=1}^\kappa \tfrac{1}{2n_k} \|\y^{k}-X^{k}\g -X^{k}L_k\|_2^2 + \lambda_L\|L\|_{1,1} \\
%= & \argmin_{L}  \sum_{k=1}^\kappa \left[ \tfrac{1}{2n_k} \|(\y^{k}-X^{k}\g) -X^{k}L_k\|_2^2 \right] + \lambda_L\|L\|_{1,1}\\
%= & \argmin_{L}  \sum_{k=1}^\kappa \tfrac{1}{2n_k} \|(\y^{k}-X^{k}\g) -X^{k}L_k\|_2^2 +  \lambda_L \sum_{k=1}^\kappa\|L_k\|_1 \\
& \argmin_{L}  \sum_{k=1}^\kappa \left[\tfrac{1}{2n_k} \|(\y^{k}-X^{k}\g) -X^{k}L_k\|_2^2 +  \lambda_L\|L_k\|_1\right]. \numberthis \label{eq:solveforL}
%= & \argmin_{L}  \tfrac{1}{2n_k} \|(\y^{k}-X^{k}\g) -X^{k}L_k\|_2^2 +  \lambda_L\|L_k\|_1.
\end{align*}
Note that each term in the sum in (\ref{eq:solveforL}) involves only one column of $L$. Therefore we can optimize each column of $L$ independently:
\begin{equation}
\argmin_{L^k}  \tfrac{1}{2n_k} \|\check{\y}^{k} - X^{k}L_k\|_2^2 +  \lambda_L\|L_k\|_1 \label{eq:solveforLk}
\end{equation}
where $X^{k}$ is the design matrix for patient $k$, $L_k$ is the column of $L$ for patient $k$ and $\check{\y}^{k} = \y^{k} - X^{k}\g$, or $\y$ adjusted for the contribution of $\g$, where $\y^{k}$ is the vector of observations for patient $k$.
Note that we are using squared loss in all of the above equations in our development, but any strictly convex loss function can be substituted, meaning that our approach applies to other generalized linear models as well. We can solve (\ref{eq:gLOP}) by choosing a starting point and alternating solving problems (\ref{eq:solveforLk}) and (\ref{eq:solveforg}) until some convergence criterion is met. Note that since (\ref{eq:gLOP}) is convex, this procedure will converge to a global optimum, although that optimum may not be unique as we discuss below.
%information about complexity here

%\begin{algorithm}[H]
%\caption{gLOP Alternating Coordinate Minimization Algorithm} \label{gLOPpseudocode}
%\begin{algorithmic}[1]
%\State {\bf Given} data $X^{k}$ (design matrix for patient $k$), $y^{k}$ (target vector for patient $k$) for patients $k = 1..\kappa$, $g_\text{init}$ (initial global model), $L_\text{init}$ (initial local model)
%\State $g_\text{new} = g_\text{init}$
%\State $L_\text{new} = L_\text{init}$
%\State For all $X^{k}$ concatenate into $X^{\ast}$
%\While{not converged}
%\State $g_\text{old} = g_\text{new}$
%\State $L_\text{old} = L_\text{new}$
%\For {$k = \{1,...,\kappa\}$}
%\State Set $\tilde{y}^{k} = y^{k} - X^{k}L_{\text{old},k}$
%\EndFor
%\State For all $\tilde{y}^{k}$ concatenate into $\tilde{y}^{\ast}$
%\State Update $g_\text{new}$ with $\text{lasso}(X^{\ast},\tilde{y}^{\ast},\lambda_g)$
%\For {$k = \{1,...,\kappa\}$}
%\State Set $\check{y}^{k} = y^{k} - X^{k}g_{\text{new}}$
%\State Update $L_{\text{new},k}$ with $\text{lasso}(X^{k},\check{y}^{k},\lambda_L)$
%\EndFor
%\EndWhile
%\end{algorithmic}
%\end{algorithm}

\subsection{A Single-Lasso View of gLOP}
\newcommand{\altparam}{\ensuremath{\boldsymbol{\xi}}}
\newcommand{\altparamscl}{\ensuremath{\xi}}
The block coordinate minimization algorithm presented above works very well and is space-efficient, taking $\mathcal{O}(\sum_k n_k p)$ space, which allows gLOP to be applied to large datasets. However it only produces the gLOP estimate for a single pair of $\lambda_{\g}$ and $\lambda_L$. We now present an alternative formulation that allows us to recover the entire regularization path for gLOP, and also allows us to characterize the uniqueness of gLOP estimates. For simplicity, we assume the $n_k$ are all equal. We define a $n \cdot \kappa\times p \cdot(\kappa+1)$ block matrix with the first $p$ columns containing the vertical concatenation of the design matrices for each patient, horizontally concatenated with a block matrix with design matrices for each patient on the diagonal. We then define target and coefficient vectors $\y$ and $\bbeta$ as follows\footnote{If the $n_k$ are not all equal, we need additional normalization in the definitions of $X$, $\y$, and $\bbeta$; we omit these details for clarity}:
\[ 
X = 
\begin{bmatrix}
X^{1} & X^{1} & 0 & \cdots & 0\\
X^{2} & 0 & X^{2} & \cdots & 0\\
\vdots & \vdots & \vdots & \ddots & \vdots\\
X^{\kappa} & 0 & 0 & \cdots & X^{\kappa}\\
\end{bmatrix},~~
\y =
\left[ \begin{array}{cc} 
\y^{1}\\
\y^{2}\\
\vdots \\
\y^{\kappa}\\
\end{array} \right],~~
\bbeta =
\left[ \begin{array}{cc} 
\g\\
L_{1}\\
L_{2}\\
\vdots \\
L_{\kappa}\\
\end{array} \right]
\]
We can then write the gLOP optimization as
\begin{equation}
\argmin_{\bbeta} \tfrac{1}{2n}\|\y - X\bbeta\|^2_2 + \sum_{i = 1}^{p \cdot (\kappa+1)} \lambda_i|\beta_i|\label{eq:weightedlasso}
\end{equation}
where we set $\lambda_i = \lambda_{\g}$ for $1 \le i \le p$ and $\lambda_i = \lambda_L$ for $p+1 \le i \le p\cdot(\kappa+1)$.
Equation (\ref{eq:weightedlasso}) is a lasso-like problem but with different regularization weights applied to different coefficients. Note that if we set $\altparamscl_i = \lambda_i \beta_i$, we may re-write the problem equivalently as
\begin{equation}
\argmin_{\altparam} \tfrac{1}{2n}\|\y - \bar{X}\altparam\|^2_2 + \|\altparam\|_1\label{eq:lasso}
, ~\mathrm{where}~~
\bar{X} = 
\begin{bmatrix}
\frac{X^{1}}{\lambda_{\g}} & \frac{X^{1}}{\lambda_L} & 0 & \cdots & 0 \\[0.4em]
\frac{X^{2}}{\lambda_{\g}} & 0 & \frac{X^{2}}{\lambda_L} & \cdots & 0\\
\vdots & \vdots & \vdots & \ddots & \vdots\\
\frac{X^{\kappa}}{\lambda_{\g}} & 0 & 0 & \cdots & \frac{X^{\kappa}}{\lambda_L}\\
\end{bmatrix}.
\end{equation}
by absorbing the penalty parameters into the design matrix. This is a standard lasso problem with $\lambda = 1$, which we can solve and then recover the original estimates by defining $\hat{\beta}_i = \hat{\altparamscl}_i / \lambda_i$ for $\lambda_i > 0$.  This formulation gives us the ability to use properties of the lasso to develop optimization algorithms for and determine properties of the gLOP estimator. In the following sections, we develop an algorithm for the full regularization path of gLOP, and we give a complete characterization of the uniqueness of gLOP estimates.

\subsection{The Full Regularization Path of gLOP}
\textit{Least-Angle Regression} (LARS) \citep{efron2004} was originally formulated as a non-greedy forward variable selection algorithm for least-squares regression. With a slight modification, the LARS algorithm produces the entire regularization path for the lasso; that is, it produces the lasso solutions for a given problem for all $\lambda \ge 0$ by expressing the solution path as a piecewise linear vector-valued function. In the following, we use ``LARS'' to mean the commonly-used lasso version of the LARS algorithm.

Using the formulation given in (\ref{eq:lasso}), we can apply the LARS algorithm to obtain a solution path for gLOP as we now show. By fixing ``reference'' values $\lambda^*_{\g}$ and $\lambda^*_L$, we apply LARS to the problem defined in (\ref{eq:lasso}), which under certain conditions (discussed further below) gives all solutions of the form
$
\hat\altparam(\lambda) = \argmin_{\altparam} \tfrac{1}{2n}\|\y - \bar{X}\altparam\|^2_2 + \lambda\|\altparam\|_1\label{eq:gLOPLARS}
$
where $\hat\altparam(\lambda)$ gives the gLOP solution for $\lambda_{\g} = \lambda\cdot\lambda^*_{\g}$, and $\lambda_L = \lambda\cdot\lambda^*_L$. Thus we obtain all solutions corresponding to a fixed ratio $\lambda_{\g}/\lambda_L$, but whose ``overall'' amount of penalization varies according to $\lambda$. This may be interpreted as fixing the amount of preference for a more global versus a more local model, and varying the total number of parameters that are permitted to enter the model. The corresponding $\hat\bbeta$ parameter is given by $\hat\beta_i(\lambda) = \hat\altparamscl_i(\lambda) / \lambda^*_{\g}$ for $i \le p$ (the global parameters) and $\hat\beta_i(\lambda) = \hat\altparamscl_i(\lambda) / \lambda^*_{L}$ for $i > p$ (the local parameters.)

\subsection{Tuning Parameters and Uniqueness}
\renewcommand{\Re}{\mathbb{R}}
\newcommand{\Aset}{\ensuremath{\mathcal{A}}}
\newcommand{\Coe}{\ensuremath{\boldsymbol\xi}}
\newcommand{\coe}{\ensuremath{\xi}}
\newcommand{\Des}{\bar{X}}
\newcommand{\tgt}{\ensuremath{\y}}
\newcommand{\Subg}{\ensuremath{\boldsymbol\alpha}}
\newcommand{\subg}{\ensuremath{\alpha}}
\newcommand{\DesAset}[1]{\ensuremath{\Des_{\Aset,#1}}}
\newcommand{\subgAset}[1]{\ensuremath{\subg_{\Aset,#1}}}
We now characterize the uniqueness of gLOP estimates in terms of the data and the tuning parameters. Proofs of the lemmas are given in Appendix A.
We begin by extending work by Ryan Tibshirani
(\citeyear{tibshirani2013}) to the case where not all lasso
coefficients are penalized equally by using our modified design matrix $\Des$.
Our objective in (\ref{eq:gLOPLARS}) has
subgradients evaluated at a point $(\Coe,\lambda)$ given by

\begin{equation}
-\Des^\T(\tgt - \Des\Coe) + \lambda \Subg, \mbox{~~s.t.~~} \subg_i \in 
\begin{cases}
\{ \sgn(\coe_i) \} &\mbox{if } \coe_i \ne 0\\\
[ -1, 1] &\mbox{if } \coe_i = 0. 
\end{cases}
\label{eq:subgradset}
\end{equation}

First-order optimality conditions state that if a point $\hat\Coe$ is optimal, then there exist $\hat\subg_i$ satisfying constraints in (\ref{eq:subgradset}) for which the given expression is the zero vector. Because our objective is convex (but not strictly convex) these conditions are also sufficient for optimality, though there is no guarantee that the optimum is unique. For an optimal point $\hat\Coe$ with corresponding $\hat\Subg$ satisfying first-order optimality conditions, define
$
\Aset = \{i \in \{1,...,p\} : |\hat\subg_i| = 1 \}
$
where $\Des_i$ is the $i$th column of $\Des$. This set contains all the indices of the non-zero coefficients, plus possibly indices of some of the zero coefficients. Let $\Des_\Aset$ be the sub-matrix of $\Des$ consisting of only the
columns with indices in $\Aset$. 

\begin{lemma}\label{lm:fullrank}
If $\Des_\Aset$ has full rank then the lasso has a unique solution.
\end{lemma}

Thus, one way of ensuring uniqueness of $\hat\Coe$ is to guarantee that $\Des_\Aset$ is of full column rank. However, we do not know which columns of $\Des$ will be in $\Des_\Aset$ until we solve the optimization problem. We could ensure full rank of $\Des_\Aset$ by guaranteeing that $\Des$ is of full rank, but we do not wish to restrict our attention to full-rank design matrices; indeed one of the most useful patients for lasso-type methods is when $p > n$ and $\Des$ obviously does not have full rank. Furthermore, for gLOP, the matrix $\Des$ will {\em never} be full-rank. For example, suppose $p = 4$ and $\kappa = 3$; then by definition we have $\Des_1 = \Des_5 + \Des_9 + \Des_{13}$, establishing that $\Des$ has linear dependence among its columns and is not full-rank.

We now show that a weaker condition -- {\it affine independence with negation} -- among columns of $\Des$ is sufficient to ensure that $\Des_\Aset$ is always of full rank. 

\begin{definition}[Affinely Independent with Negation (AIN)] The columns of an $n \times p$ matrix $X$ are {\em Affinely Independent with Negation} (AIN) if there {\em do not exist} signs $s_i$, weights $w_i$, and an index $j$ such that
$X_j = \sum_{\stackrel{i=1}{i\ne j}}^p w_i s_i X_i
$,
where $s_i \in \{-1,1\}$, and $\sum_i w_i = 1$. (Standard affine independence is similar but does not allow for the $s_i$.)
\end{definition}

\begin{lemma}\label{lm:AIN}
If the columns of $\Des$ are AIN, then $\Des_\Aset$ has full rank.
\end{lemma}

The design matrix $\Des$ is defined in terms of the ``original'' gLOP matrix $X$ together with $\lambda_{\g}$ and $\lambda_L$. Each of these components influences whether or not the columns of $\Des$ will be AIN; we summarize this in the following theorem.
\begin{theorem}\label{thm:unique}
If the the columns of each matrix $X^{k}$, $k = 1...\kappa$ are AIN, and if $\lambda_L > \lambda_{\g}$, then the columns of $\Des$ are AIN and there is a unique gLOP solution.
\end{theorem}
\begin{proof}
First, we note that columns from different patient blocks are orthogonal; thus we cannot construct columns from one patient block using weighted sums of columns from other patient blocks.
Furthermore, we assume that within each patient block we have AIN columns. (We will discuss this assumption more later.) Therefore, if there is affine dependence among columns, it must involve the patient blocks and the global block.

Each column in the global block of $\Des$ can be written as a linear combination of columns from the patient blocks. In particular, 
for $i \le p$, we have $X_i/\lambda_{\g} = \sum_{k=1}^\kappa \left( \frac{\lambda_L}{\lambda_{\g}} \right) X_{p\cdot{}k+i}/\lambda_L.
$
We could write this linear combination with weights $w_i = \frac{\lambda_L}{\lambda_{\g}}$ and signs $s_i = 1$, or we could also negate any number of the $s_i$ along with their corresponding $w_i$ to achieve the same result. If we negate $k$ of the columns, the sum of the weights is given by $(\kappa - k)\frac{\lambda_L}{\lambda_{\g}} - k ({\lambda_L}/{\lambda_{\g}})$.
Therefore to ensure the AIN property of $\Des$, we may choose $\lambda_{\g}$ and $\lambda_L$ such that
$(\kappa - k)({\lambda_L}/{\lambda_{\g}}) - k ({\lambda_L}/{\lambda_{\g}}) \ne 1$,
or equivalently
$\left(\kappa - 2k\right) \frac{\lambda_L}{\lambda_{\g}} \ne 1 $
for all $0 \le k \le \kappa$. We can ensure that this holds by noting that if we choose $\lambda_L > \lambda_{\g}$, then the inequality holds because $(\kappa - 2k)$ is an integer and $\frac{\lambda_L}{\lambda_{\g}} > 1$, so their product cannot possibly equal 1. Note that if $\kappa$ is even, then $\frac{\lambda_L}{\lambda_{\g}} > \tfrac{1}{2}$ is sufficient.
\end{proof}

Theorem \ref{thm:unique} characterizes uniqueness of gLOP in terms of the penalty parameters and the data matrices for each patient, $X^1$ through $X^k$. \citet{tibshirani2013} note that a design matrix drawn from a continuous probability distribution on $\Re^{np}$ has the AIN property with probability one \citep[see][Lemma 4]{tibshirani2013}. Thus, for matrices of continuous feature values, the uniqueness of gLOP can be assured by an appropriate choice of $\lambda_{\g}$ and $\lambda_L$ alone. Design matrices containing discrete entries require more careful analysis to ensure uniqueness. Another interesting consequence of this is that all of our lasso sub-problems for $\g$ and the $L_k$ that are used as part of our block coordinate minimization algorithms have unique solutions under the same conditions on the $X^k$, even if the joint minimization does not have a unique solution. Finally, we note that \citet{jalali2010} provide results addressing the uniqueness of the dirty model estimates asymptotically in $n$ with high probability rather than with probability one as we have here.

\subsection{Empirical Results: Synthetic Data}
%\subsubsection{Experimental Setup}
\newcommand{\Truep}{\ensuremath{\boldsymbol{\theta}}}
\newcommand{\truep}{\ensuremath{\boldsymbol{\theta}}} 
\newcommand{\Trues}{\ensuremath{\boldsymbol{\tau}}}
\newcommand{\trues}{\ensuremath{\tau}}

%\subsubsection{Empirical Results}
To evaluate the {\em in-population} predictive accuracy of gLOP and the dirty model in different contexts, we conduct four main experiments using different sizes of $p$ and $\kappa$. By {\em in-population}, we mean accuracy on future data gathered on the same population of patients as were used for training. This is contrasted with {\em out-of-population} prediction, when a model is used to predict labels for instances of future (unseen) patients.
For each experiment, we run 100 trials using data generated as described above and average the results.

Means and standard deviations of the MSE from each experiment are shown in Table \ref{table:error}; for comparison, we also include the MSE of the result achieved by the standard lasso, which essentially ignores the distinction between data from different patients. The detailed experimental setup is given in Appendix B. In all cases, the test error for gLOP is statistically significantly lower than the error for the dirty model and for the lasso. In the small-$p$ settings, both gLOP and the dirty model do much better than the lasso because they allow different patients to have different coefficients, resulting in much better model fit. 
%For a fixed $p$, as we increase $\kappa$, we note that that for the small-$p$ case the gLOP results are not statistically significantly different, whereas the error of the dirty model decreases. We conjecture that gLOP is approaching the theoretical best MSE, which would be 1.0 (the noise level in our experiments.) 
For the large-$p$ case, the error for gLOP decreases with $\kappa$, but the error for the dirty model remains the same. We attribute this to the increased number of parameters that the dirty model tries to learn; that is, $2\cdot \kappa \cdot p$ parameters for the dirty model versus $p + \kappa \cdot p$ for gLOP, even though they have the same representational power.

We also conducted a simple experiment to illustrate how gLOP can identify predictive outliers. Consider a case where we have data from several patients and we want to construct a global model, but 20 percent of patients are predictive outliers in the sense that they have very different local intercepts from the remainder of the patients. The true model is given by $Y = 1 + X + cZ + \epsilon$, where $X \sim \mathcal{N}(0,1)$, $Z \sim \text{Bernoulli}(0.2)$, $\epsilon ~ \mathcal{N}(0,1)$ and $c=10$.  In this case, we can use gLOP to construct a model {\em using only $X$ and $Y$} to identify predictive outliers by examining local coefficients. 
For our example, we generated synthetic data for $\kappa$ = 16 patients, $n_k$ = 10 observations per patient, and $p$ = 32 features.
Using only $X$ and $Y$, gLOP correctly identified the 5 outliers (for whom $Z = 1$) by assigning them larger non-zero local intercepts. Adding $Z$ into the model results in improved performance of the global model and no detection of predictive outliers. In a clinical setting, we envision that gLOP could be used to first identify predictive outliers, which would then direct the search for a new feature that could identify them and in turn improve the predictions of the global model.

\section{Empirical Results: Health Research Data}
We now present two examples of how gLOP can be used to identify predictive outliers and direct future data gathering to improve predictive performance.
The Sequenced Treatment Alternatives to Relieve Depression (STAR$^*$D) study was a multi-stage, multi-centre prospective randomized clinical trial assessing interventions for non-psychotic major depressive disorder (MDD) in the context of sequential alternatives for MDD treatment \citep{rush04sequenced}. In the study, proceeding to a new treatment was contingent on the clinician-rated Quick Inventory for Depressive Symptomatology (QIDS; \cite{Rush2003}; lower scores are better). If a patient did not improve according to this scale, a new treatment was initiated. According to the protocol, self-reported QIDS was also recorded at each clinic visit. We used gLOP to predict clinician-rated QIDS from self-report QIDS using data from 1,368 patients over 15,593 visits, using patient age as an additional demographic variable in $g$ only.
$\lambda$ values were chosen using the Bayesian Information Criterion (BIC) \citep{zou2007} as for some patients there were not enough observations to support cross-validation.

The global intercept and slope revealed that patients on average tended to rate their own symptoms lower than clinicians did. We found that 25 patients had a local intercept, 32 patients had a nonzero local slope coefficient, and 2 patients had both a local intercept and slope. Of the patients that had a local intercept, all but 3 had a positive coefficient indicating that this group of patients tended to underrate their symptoms even more than the general population relative to the clinician-rated scores. Of the patients that had a nonzero local slope, all but 4 had negative coefficients indicating that the more severe their symptoms at a given time, the more likely they are to rate themselves more severely in relation to the clinician QIDS scores. To improve the global model, we would suggest searching for features that might identify this minority of patients who would tend to rate their symptoms lower on average than their peers; such features might be identified using theory and expertise in the study of major depressive disorder.

The Oxford Parkinson's Disease Telemonitoring Dataset comprises a collection of speech signals collected from 42 people (5,875 observations in total) with early-stage Parkinson's Disease (PD) collected during a telemonitoring study to assess progression of PD symptoms remotely using speech characteristics \citep{tsanas2010accurate,tsanas2010enhanced}. 
%The dataset consists of 16 measures of speech characteristics (such as Shimmer, Jitter, etc; see \cite{tsanas2010enhanced}), two measures of PD severity using the Unified Parkinson's Disease Rating Scale (UPDRS), and individual characteristics such as age, with multiple observations present per patient. 
Previous studies using these dysphonia measures have been able to both distinguish persons with PD from healthy subjects \citep{little2009}, and to predict PD symptom severity remotely using linear and non-linear regression techniques \citep{tsanas2010enhanced}. 
%We conducted experiments using R, with the LARS package \cite{lars} as a subroutine in the optimization algorithm. 
%The original data set was sampled into a test data-set and a training cross-validation training set (40 samples per individual in each set). 
%10-fold cross validation was used to choose values for $\lambda_g$ and $\lambda_L$.
%Previous work with this dataset has focused on predictive model building, in the form of identifying sets of optimal predictors for both distinguishing healthy individuals from individuals with PD \cite{little2009}, and identifying PD symptom severity for remote progression monitoring \cite{tsanas2010enhanced}. 
We used gLOP to predict the total Unified Parkinson's Disease Rating Scale (UPDRS) score based on waveform features, using BIC to choose penalization parameters.
%The goal of our analysis was to discover the underlying etiological structure of the data by identifying which features impact predicted severity in the same way for all patients, and which features impact predicted severity for each patient differently. 
We included a penalized local intercept for each patient, allowing us to capture variability in average PD symptom severity between patients and to interpret the remaining coefficients as the influence of each feature on the departure from a patient-specific mean severity.

We found that \textit{MDVP absolute jitter}, \textit{MDVP local shimmer (dB)}, \textit{eleven point amplitude perturbation quotient}, \textit{noise-to-harmonics ratio}, \textit{harmonics-to-noise ratio}, \textit{recurrence period density entropy}, \textit{detrended fluctuation analysis}, and \textit{pitch period entropy} were globally predictive; the inclusion of the latter four features is consistent with previous work on classification \citep{little2009}.
The coefficient matrix $L$ was mainly sparse, but all but one feature had one or more local coefficients. 
Of the 42 patients, 5 appear to be outliers, as the sum of the absolute values of their local coefficients were above the 90th percentile of all values calculated.
Based on this, it is possible that including additional features in the model could help to distinguish these patients and future similar patients from the bulk of patients for whom the global model predicts well. 
%The goal of further data collection stemming from this exploratory analysis would be to find an additional feature or features that accounts for the additional variance from the outliers, making the sample more homogeneous locally with respect to the global model. 
%; alternatively, it is possible that the inclusion of more samples per patient would increase the precision of the global model to the point where the number of local coefficients would be greatly decreased. 
%Experimental results are shown in Appendix C.
%The coefficient matrix $L$ was mainly sparse, with local coefficients for only two patients in the HNR and DFA features. Hence, the current data do not support the existence of many predictive outliers, and it is likely that gathering more training examples would be required to improve the model.
%If a high level of homogeneity had been found between patients in the dataset, then moving forward this analysis could have been used to direct scientific inquiry by grouping patients according to their local coefficients and gathering more data to investigate the reasons for similarities among grouped patients. 
%Given this low observed heterogeneity, we hypothesize that the global model may be a good choice for prediction with this group of patients. However, two patients also have deviations from the full model; the cause of this could be explored using additional collected data.

%it could potentially be interesting to collect additional information to determine why this might be the case. 
%, since very few patients required local coefficients. 
\section{Conclusion and Future Work}
The gLOP model lays conceptual and methodological groundwork for capturing predictive heterogeneity by identifying predictive outliers, which can help direct future data-gathering activities in cases where data collection may be difficult, invasive, or costly. We may in future want to impose additional constraints on $L$ to allow for only a small number of types of outliers that have the same local coefficients; this could be achieved by shrinking columns of $L$ toward each other, similar in concept to the fused lasso \citep{tibshirani05fused}. Using this approach, we could match new patients to one of a few local ``subtypes'' of patient in order to achieve better predictions. Another direction for future research is to incorporate post selection inference into gLOP as has been done with LARS \citep{taylor2014}, which would provide additional confidence information. Finally, we aim to incorporate gLOP into a visual exploratory data analytics system
% take this out if necessary 
 that will reveal predictive outliers and other kinds of hidden structure in datasets used for predictive modelling.

%\section{Conclusion} 

%We have described the gLOP (global/LOcal Penalty) model for multi-task learning that is capable of using data from different patients to build predictive models. We have shown that gLOP is easier both to interpret and to optimize than previous related models by presenting two optimization algorithms for gLOP, one based on block coordinate minimization and the other that provides the full regularization path for gLOP. We have also presented a characterization of the uniqueness of the gLOP estimator in terms of the data and relevant tuning parameters, and we have presented experimental results showing that gLOP can outperform the dirty model in terms of predictive performance. Finally, we have demonstrated the use of the algorithm on two health research datasets, one with evidence for predictive heterogeneity and one without. 
% ACKNOWLEDGEMENTS ONLY GO IN THE CAMERA-READY, NOT THE SUBMISSION
% \acks{Many thanks to all collaborators and funders!}

\bibliography{mlhc2016}

\clearpage
\appendix

\section*{Appendix A}

\label{ap:proofs}

\begin{proof}[Of Lemma \ref{lm:fullrank}]
The proof uses a strategy adapted from \citet{tibshirani2013} with modifications to accommodate different penalizations on different columns. Let $\hat\Coe_\Aset$ be the sub-vector
of $\hat\Coe$ containing only the elements with indices in
$\Aset$, and let $\hat\Subg_\Aset$ be the analogous sub-vector of
$\hat\Subg$. Because all coefficients not contained in $\hat\Coe_\Aset$ must be zero and do not contribute to the lasso fit, equating (\ref{eq:subgradset}) to the zero vector gives
\begin{equation}
\Des_\Aset^\T(\tgt - \Des_\Aset\hat\Coe_\Aset) = \lambda \hat\Subg_\Aset.\label{eq:subgradAset}
\end{equation}

Therefore $\lambda \hat\Subg_\Aset \in \rowspace(\Des_\Aset)$, and 
$\Des_\Aset^\T (\Des_\Aset^\T)^+ \lambda \hat\Subg_\Aset = \lambda \hat\Subg_\Aset$ where $X^+$ indicates the Moore-Penrose pseudoinverse of $X$. Further algebraic manipulation of (\ref{eq:subgradAset}) gives
\begin{align*}
%\Des_\Aset^\T(\tgt - \Des_\Aset\hat\Coe_\Aset) & = \lambda \hat\Subg_\Aset\\
%\Des_\Aset^\T \Des_\Aset\hat\Coe_\Aset & = \Des_\Aset^\T\tgt - \lambda \hat\Subg_\Aset\\
%\Des_\Aset^\T \Des_\Aset\hat\Coe_\Aset & = \Des_\Aset^\T(\tgt - (\Des_\Aset^\T)^+\lambda \hat\Subg_\Aset) \\
%(\Des_\Aset^\T)^+ \Des_\Aset^\T \Des_\Aset\hat\Coe_\Aset & = (\Des_\Aset^\T)^+ \Des_\Aset^\T(\tgt - (\Des_\Aset^\T)^+\lambda \hat\Subg_\Aset) \\
%\Des_\Aset\hat\Coe_\Aset & = (\Des_\Aset^\T)^+ \Des_\Aset^\T(\tgt - (\Des_\Aset^\T)^+\lambda \hat\Subg_\Aset) \\
%\Des_\Aset\hat\Coe_\Aset & = (\Des_\Aset^+)^\T \Des_\Aset^\T(\tgt - (\Des_\Aset^\T)^+\lambda \hat\Subg_\Aset) \\
%(\Des_\Aset \Des_\Aset^+)^\T\Des_\Aset\hat\Coe_\Aset & = (\Des_\Aset \Des_\Aset^+)^\T(\tgt - (\Des_\Aset^\T)^+\lambda \hat\Subg_\Aset) \\
\Des_\Aset\hat\Coe_\Aset & = \Des_\Aset \left[ \Des_\Aset^+(\tgt - (\Des_\Aset^\T)^+\lambda \hat\Subg_\Aset)\right].
\end{align*}
Note that $\Des_\Aset \hat \Coe_\Aset$ is the optimal lasso {\bf fit}, which is unique even though the coefficients providing that fit may not be \citep{tibshirani2013}. The set of all optimal coefficient vectors is given by
\begin{equation*}
\{\Coe_\Aset : \Coe_\Aset = \Des_\Aset^+(\tgt - (\Des_\Aset^\T)^+\lambda \hat\Subg_\Aset) + \mathbf{z}\}
\end{equation*}
for $\mathbf{z} \in \nullspace(\Des_\Aset)$ and $\sgn(\Des_\Aset^+(\tgt - (\Des_\Aset^\T)^+\lambda \hat\Subg_\Aset) + \mathbf{z}) = \hat\Subg_\Aset$.
If $\Des_\Aset$ has full rank then $\nullspace(\Des_\Aset) = \mathbf{0}$ and we have a unique optimal coefficient vector obtained by setting $\mathbf{z} = \mathbf{0}$.
\end{proof}

\begin{proof}[Of Lemma \ref{lm:AIN}]
Suppose that $\Des_\Aset$ does not have full rank. Then for some column $i$ of $\Des_\Aset$, there exist weights $c_j$ such that
\begin{equation}
\DesAset{i} = \sum_{j \ne i} c_j \DesAset{j}. \label{eq:colsum}
\end{equation}
Recall that each index has a corresponding $\subgAset{i} \in \{1,-1\}$ as defined in (\ref{eq:subgradset}). It follows that 
\begin{equation*}
\DesAset{i} = \sum_{j \ne i} (c_j \subgAset{i} \subgAset{j}) \cdot \frac{\subgAset{j}}{\subgAset{i}} \DesAset{j}.
\end{equation*}
By definition, $\frac{\subgAset{j}}{\subgAset{i}} \in \{-1,1\}$. We will now show that the weights $c_j \subgAset{i} \subgAset{j}$ sum to 1. Recall from (\ref{eq:subgradAset}) that $(\Des_{\Aset,i})^\T(\y - \Des\hat\Coe) = \lambda_i \subgAset{i}$. Therefore, using (\ref{eq:colsum}) we have
\begin{align*}
(\DesAset{i})^\T (\y - \Des\hat\Coe) & = \sum_{j \ne i} c_j (\DesAset{i}^\T) (\y - \Des\hat\Coe) \\
\subgAset{i} & = \sum_{j \ne i} c_j \subgAset{j} \\
1 & = \sum_{j \ne i} c_j  \subgAset{i} \subgAset{j}
\end{align*}
This establishes that if $\Des_\Aset$ does not have full rank, then its columns are not AIN because we can produce signs $s_i = \frac{\subgAset{j}}{\subgAset{i}}$ and weights $w_i = c_j \subgAset{i} \subgAset{j}$ as witnesses.
\end{proof}

\section*{Appendix B} \label{ap:setup}

In order to evaluate gLOP in a controlled fashion and to compare it with the dirty model, we conduct experiments with varying numbers of features and patients. We draw elements of the design matrices from a Gaussian with 0 mean and unit variance, and we generate observations $\y^k$ by adding Gaussian noise ($0$ mean, unit variance) to $X\Truep^k$ for each of $\kappa$ patients, where $\Truep^k$ gives the true model coefficients for that patient. 

First, we explored differences in the output of gLOP versus the dirty model using a small example ($p = 4, \kappa = 5,n = 64$) with the following parameters:
\begin{equation*}
\Truep^1 = \Truep^2 = \Truep^3 = (0,0,3,3)^\T,~~
\Truep^4 = (0,0,-3,3)^\T,~~\Truep^5 = (0,3,0,3)^\T
\end{equation*}
Identical data sets were used for each algorithm with $\lambda_{{\g}/B} = 5$ and $\lambda_{L/S} = 10$ was chosen by cross-validation. 
Based on this example, we observed that gLOP's $\hat{\g}+\hat{L}_k$ recovers the true model parameters for patient $k$ quite closely, although false inclusions were present in two of the patients. In contrast, the parameters recovered by the dirty model did not capture any variation between patients in $\hat S$, which was exactly zero, and induced little variation between patients in $B$. We found that in our experiments it was often impossible to find a pair of $\lambda_B$ and $\lambda_S$ such that both $B$ and $S$ contained non-zero values.

We also conduct larger-scale experiments to evaluate predictive performance; in these we use three different patient-types with true parameters $\Trues^1$, $\Trues^2$, and $\Trues^3$, given by
\begin{equation*}
\Trues^1 = (\underbrace{3, ..., 3}_{1,...,\frac{p}{4}}, \underbrace{0, ..., 0}_{\frac{p}{4}+1,...,p})^\T, 
\Trues^2 = (\underbrace{3,  -3, ..., 3, -3}_{1,...,\frac{p}{2}}, \underbrace{0,..., 0}_{\frac{p}{2}+1,...,p})^\T, 
\Trues^3 = (\underbrace{-3,  3, ..., -3, 3}_{1,...,\frac{p}{8}}, \underbrace{0,..., 0}_{\frac{p}{8}+1,...,p})^\T.
\end{equation*}
Note that 3/4 of the entries in $\Trues^1$ are zero, $\Trues^2$ is less sparse than $\Trues^1$, and $\Trues^3$ is more sparse than $\Trues^1$. 
In each experiment, $\frac{\kappa}{8}$ patients are generated using each of $\Trues_2$ and $\Trues_3$; the remaining $\kappa - \frac{\kappa}{4}$ patients are generated using $\Trues_1$. 

To choose values for $\lambda_g$ and $\lambda_L$, we perform 10-fold cross validation over a grid of points $(\lambda_{\g},\lambda_L)$ but constrain the results to include only cases where $\lambda_{\g} \le \lambda_L$. The folds are stratified across patients. The grid ranges from 0 to 100 in steps of 5 along each dimension. While this may appear coarse, note that our loss function (squared error) is not normalized by $n$ as is sometimes common, in order to compare more directly with the original dirty model implementation. Thus the range of useful $\lambda$ is wider in our experiments than in other lasso applications. Once CV error for all pairs has been calculated, the pair of $\lambda_{\g}$ and $\lambda_L$ with the lowest error is selected.  We break ties in favour of larger $\lambda_L$ and then larger $\lambda_{\g}$ in order to obtain the sparsest model. We evaluate the prediction error for each learned model using a large ($n = 1000$) held-out test set.

%In order for our results to be comparable to the results obtained by the dirty model, we used the same optimization convergence criterion as in the dirty model implementation \citep{jalaliwebsite}. Specifically, we considered the algorithm to have converged if the following inequality was satisfied:
%\[
%\lvert(m_{\text{old}}-m_{\text{new}})\rvert<n \cdot p \cdot \kappa \cdot \epsilon \cdot m_{\text{old}}
%\]
%where $\epsilon = 10^{-5}$, $m_{\text{new}}$ is the value of the objective function calculated at the current iteration and $m_{\text{old}}$ is the value of the objective function calculated at the previous iteration. 
%A more stringent criterion may be desirable depending on the application of the model (i.e.\ to real clinical data versus synthetic data). A relatively large value of $\epsilon$ was used in our experiments due to convergence difficulties of the dirty model on data sets with large $p$.

\clearpage

\section*{Appendix C}
\label{ap:detailedresults}

\begin{table}[h!]
\centering
\caption{Test error results for gLOP versus the dirty model. All errors attained by the dirty model and the lasso were statistically significantly worse than those of gLOP ($\textit{p} < 0.05$ on an independent two-sample {\it t}-test).}
%\begin{table}
\begin{tabular}{llllll}
  \toprule
  $p$ & $\kappa$ & $n$ & gLOP    & Dirty M. & Lasso\\ \hline
  16  & 16  & 64 & 1.3931  & 6.3718     & 39.5652\\
    &     & & $\pm$0.0637  & $\pm$0.2599  &    $\pm$0.5335 \\ \hline 
16  & 128 & 64 & 1.4602  & 2.2171    & 39.8316  \\
    &     & & $\pm$0.0237  & $\pm$0.121  &  $\pm$0.1668  \\ \hline
128 & 16  & 64 & 93.6959 & 141.1617  & 306.5222 \\
    &     & & $\pm$7.5097  & $\pm$9.5854  &  $\pm$ 3.6947 \\ \hline
128 & 128 & 64 & 73.9881 &  141.1624     &    307.3203  \\
    &     & & $\pm$2.7155  &  $\pm$3.7506  &  $\pm$1.2835 \\
\bottomrule
\end{tabular}
\label{table:error}
%\end{table}
\end{table}

%\begin{table}[h!]
%\centering
%\caption{STAR*D demographic variables summarized by sign of local intercept.}
%\label{my-label}
%\begin{tabular}{lllll}
%\hline
%            &                  & \multicolumn{3}{l}{Local Intercept Coefficient} \\
%Demographic &                  & Negative        & Zero         & Positive       \\ \hline
%Age         & Mean             & 39.78           & 41.04        & 43.51          \\
%            & SD               & 11.11           & 13.19        & 12.34          \\ \hline
%Schooling   & Mean             & 12.50           & 13.74        & 13.04          \\
%            & SD               & 2.34            & 3.22         & 3.41           \\ \hline
%Sex         & Female           & 67.11           & 61.78        & 58.67          \\
%            & Male             & 32.89           & 38.22        & 41.33          \\ \hline
%Race        & Not White        & 27.63           & 18.11        & 22.67          \\
%            & White            & 72.37           & 81.89        & 77.33          \\ \hline
%Employment  & Unemployed       & 46.05           & 34.87        & 46.67          \\
%            & Employed/Retired & 53.95           & 65.13        & 53.33          \\ \hline
%\end{tabular}
%\label{table:stard}
%\end{table}

\end{document}